\documentclass{ecai}

\usepackage{latexsym}
\usepackage{amssymb}
\usepackage{amsmath}
\usepackage{amsthm}
\usepackage{booktabs}
\usepackage{enumitem}
\usepackage{graphicx}
\usepackage{color}

\usepackage{multirow}
\usepackage{mathtools}
\usepackage{algorithm}
\usepackage{algorithmic}
\usepackage{hyperref}
\usepackage{subfigure}
\usepackage{subcaption}

\newtheorem{theorem}{Theorem}
\newtheorem{lemma}[theorem]{Lemma}
\newtheorem{assumption}{Assumption}
\newtheorem{remark}{Remark}

\newcommand{\BibTeX}{B\kern-.05em{\sc i\kern-.025em b}\kern-.08em\TeX}

\begin{document}

\begin{frontmatter}

\title{Mask-Encoded Sparsification: Mitigating Biased Gradients in Communication-Efficient Split Learning}

\author[1]
{Wenxuan Zhou}

\author[1]
{Zhihao Qu\footnote[\dagger]{Corresponding author.}}

\author[1,2]
{Shen-Huan Lyu}

\author[3]
{Miao Cai}

\author[2]
{Baoliu Ye}

\address[1]{Key Laboratory of Water Big Data Technology of Ministry of Water Resources, Hohai University, Nanjing, China}

\address[2]{National Key Laboratory for Novel Software Technology, Nanjing University, Nanjing, China}

\address[3]{Nanjing University of Aeronautics and Astronautics, Nanjing, China}

\address{\texttt{\{zhouwx,quzhihao,lvsh\}@hhu.edu.cn} \quad\texttt{miaocai@nuaa.edu.cn}\quad\texttt{yebl@nju.edu.cn}}

\begin{abstract}
This paper introduces a novel framework designed to achieve a high compression ratio in Split Learning (SL) scenarios where resource-constrained devices are involved in large-scale model training. Our investigations demonstrate that compressing feature maps within SL leads to biased gradients that can negatively impact the convergence rates and diminish the generalization capabilities of the resulting models. Our theoretical analysis provides insights into how compression errors critically hinder SL performance, which previous methodologies underestimate. To address these challenges, we employ a narrow bit-width encoded mask to compensate for the sparsification error without increasing the order of time complexity. Supported by rigorous theoretical analysis, our framework significantly reduces compression errors and accelerates the convergence. Extensive experiments also verify that our method outperforms existing solutions regarding training efficiency and communication complexity. Our code can be found at  \url{https://github.com/BinaryMus/MaskSparsification}.
\end{abstract}
\end{frontmatter}

\section{Introduction}
In recent years, there has been an upsurge in demand for training deep neural networks (DNNs)~\cite{dosovitskiy2021an, ouyang2022training, wang2023scientific} on limited-resource end devices such as smartphones and embedded AI chips, primarily for edge intelligence applications such as speech recognition and intelligent healthcare. 
Federated learning (FL)~\cite{pmlr-v202-li23z, pmlr-v202-song23h}, a widely-used distributed learning approach, allows a group of end devices to train a global model collaboratively without sharing personal data. 
However, FL requires training the entire DNN on these local devices before aggregating the results on a central server, which can lead to server underutilization and strain on these devices' computational and energy resources. 
Moreover, as DNNs become larger and more complex~\cite{krause2023commonsense, kruger2023performance, ouyang2022training}, this issue becomes more acute and potentially limits the feasibility of FL in resource-constrained environments.

To solve this problem, Split Learning (SL)~\cite{gupta2018distributed, 10314792, kang2017neurosurgeon, 10129922,10.1145/3538641.3561500} is increasingly emerging as a promising approach that leverages the capabilities of both the server and end devices. 
Specifically, SL splits the DNN models into two parts, namely the client-side model and the server-side model, which are deployed on multiple end devices and the server, respectively.
During forward propagation, intermediate feature maps (also called \textit{smashed data}) from the cutting position (also called \textit{cutlayer}) are transmitted from end devices to the server to facilitate training on the server-side layers. 
In backpropagation, partial derivative data from the cutlayer is distributed to the end devices to update client-side models. 
Consequently, the SL paradigm enables resource-limited end devices to participate in large-scale DNN training while protecting the privacy of both local data and the global model.

One of the primary challenges in SL is the high communication overhead~\cite{zhou2022hierarchical} that results from transmitting feature maps, particularly during forward propagation. 
In FL, lazy aggregation~\cite{mcmahan2017communication} can linearly reduce communication overhead. However, due to the non-additive nature of feature maps, lazy aggregation cannot be directly applied to SL. 
As a result, iterative transmissions in the SL paradigm consume significant time and are unaffordable due to limited bandwidth connections between the end devices and the server. In addition, compressing the feature map during the training phase enables the neural network to adapt to errors, reducing communication overhead during the inference phase.
Therefore, it is essential to develop communication-efficient techniques for SL to enable its practical deployment in the edge environment.

Various technologies have been investigated to reduce communication complexity in FL and other distributed learning paradigms, typically including quantization~\cite{pmlr-v202-li23o, pmlr-v202-wang23t} and sparsification~\cite{pmlr-v202-wang23t, Zhou_2023_ICCV}. 
However, these approaches are unsuitable for SL because they are designed to compress gradients or model parameters. 
Firstly, compressing feature maps, even with unbiased compression, will inevitably introduce biased gradients, thus affecting the model's convergence~\cite{beznosikov2023biased, stich2020analysis}. 
Secondly, compression error is a crucial determinant of the model's convergence. 
However, quantization and sparsification methods typically result in massive compression errors because quantization is sensitive to outliers (top values in the feature map), and sparsification produces errors due to filtered values. 
Finally, since feature maps from different training samples are independent, existing compensation mechanisms~\cite{NEURIPS2022_6fb9ea51,NEURIPS2021_ff1ced30}, such as reserving and adding filter-out values after Top-$k$ sparsification for the subsequent transmission, are inapplicable.

To overcome the abovementioned challenges, we propose a new communication compression framework for SL: mask-encoded sparsification (MS). MS is based on Top-$k$ sparsification and uses an encoded mask to compensate for the significant compression errors caused by sparsification. 
Specifically, the feature map is first subject to a Top-$k$ sparsification, after which an encoded mask indicates the position and value of every element in the feature map. The mask of all top values is a binary number of all $1$s, and the remaining masks are calculated with a quantization-like method.
As depicted in Fig.~\ref{sl}, MS utilizes a narrow bit width mask ($2$-bit), effectively reducing the compression error caused by sparsification and being helpful for convergence.

\begin{figure*}[htbp]
    \centering
    \includegraphics[width=0.7\textwidth]{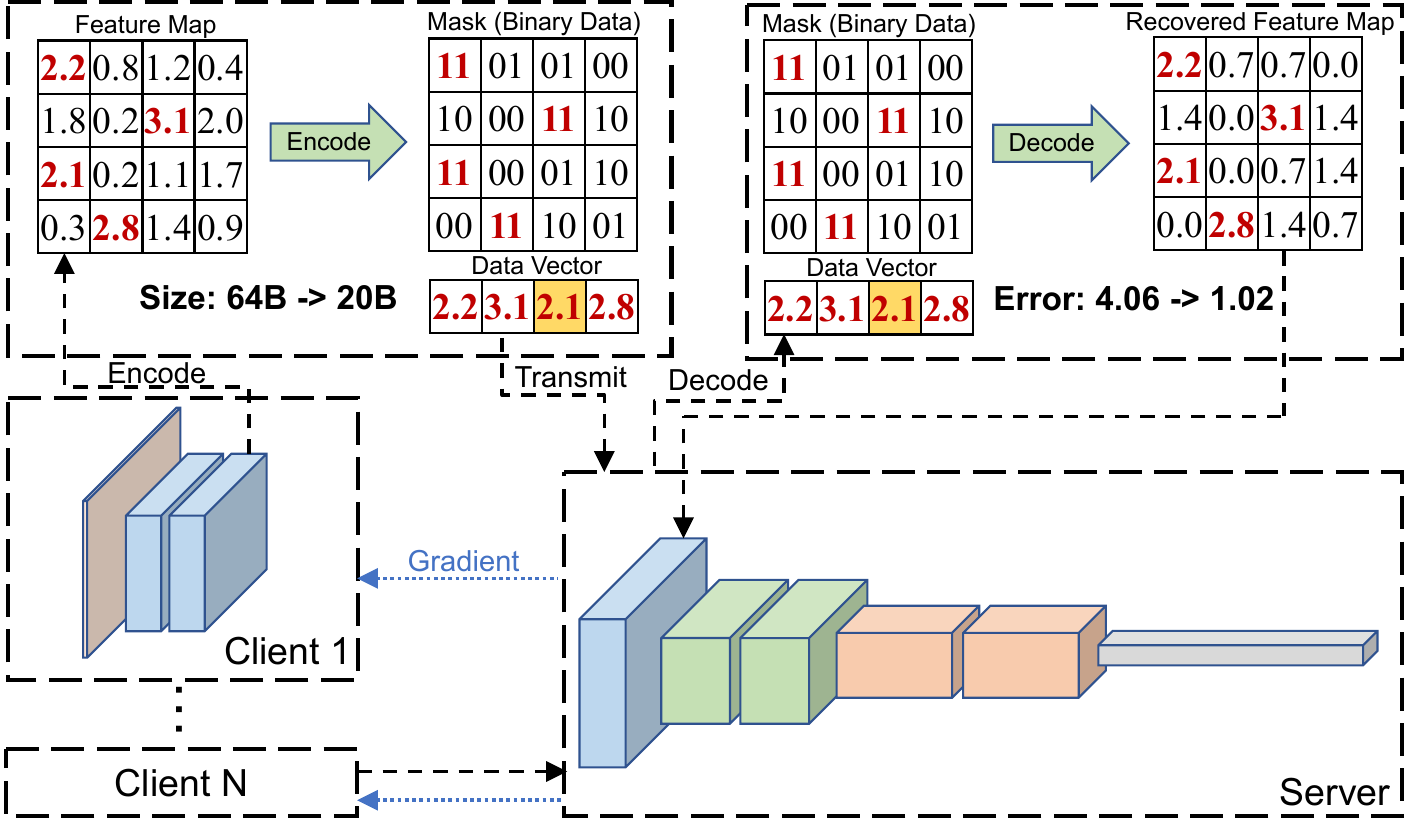}
    \caption{Using MS algorithm to compress smashed data in SL. Perform Top-$4$ sparsification on the smashed data, then compensate for filtered data with a $2$-bit width mask. The top values (highlighted in red) and will form a data vector in their original order. Their corresponding mask is $0b11$. The filtered values are calculated with a quantization-like method, MS evenly divides $0$ to $2.1$ (the smallest top value, with a yellow background) into $3$ ($3=2^2-1$) segments ($[0, 0.7)$,$[0.7, 1.4)$, and $[1.4, 2.1)$), the value in each segment will be rounded to $0$, $0.7$, and $1.4$, and their masks are $0b00$, $0b01$, and $0b10$, respectively. 
    Finally, the data vector is sent to the server along with the mask for decoding. All values with a mask of $0b11$ will be sequentially filled with the Top value, while the remaining masks will be filled with $0$, $0.7$, and $1.4$, respectively.
    This strategy markedly diminishes compression errors (measured by the $\ell_2$ norm) from $4.06$ to $1.02$, achieving superior performance compared to vanilla sparsification.}
    \label{sl}
\end{figure*}

The main contributions of this work are summarized as follows:
\begin{itemize}
    \item Our study reveals that compressing feature maps leads to biased gradients in SL. We establish the convergence rate of SL with compressed feature maps and find a situation where gradients converge to $0$ at the same rate as in the uncompressed version. 
    \item Our theoretical findings indicate that lower compression errors lead to improved convergence. Therefore, we propose mask-encoded sparsification (MS) and theoretically demonstrate its superior performance over previous methods under mild conditions. Also, empirical validation confirms that MS achieves lower errors than previous approaches.
    \item We conduct extensive experiments on various models and datasets. Experimental results show that MS significantly reduces communication overhead without accuracy loss, while previous methods deteriorate convergence or generalization. Notably, according to the empirical study on different cutlayers, we find that the feature extraction layers of neural networks (i.e., shallow layers) are more sensitive to compression errors.
\end{itemize}

\section{Related Work}
\subsection{Split Learning}

Split Learning (SL)~\cite{10274134, 10229027, 10314792, 10234566} partitions a DNN into client-side and server-side model~\cite{thapa2022splitfed, vepakomma2018split, 10129922, gupta2018distributed}, which are deployed on multiple end devices and the server, respectively.
This strategy addresses not only the efficient utilization of distributed computing power but also privacy concerns related to raw data and the sharing of the global model~\cite{vepakomma2018split}, as it allows clients to pre-process their data locally and retain a part of the model. 
Despite these benefits, SL still encounters the challenge of communication overhead~\cite{rt, zhou2022hierarchical}. 
The frequent exchange of local feature maps between the client and server leads to significant network traffic, especially in wireless networks, which can negatively impact training efficiency.

\subsection{Communication Compression}

\paragraph{Biased Compression Techniques}
Quantization~\cite{pmlr-v202-dorfman23a, 10228970} aims to minimize the bit width of numerical values by transforming them from their "continuous" forms to discrete counterparts, thus reducing the data storage demands. 
However, quantization inherently exhibits sensitivity toward outliers, leading to potentially significant compression errors. 
Methods based on clipping~\cite{kim2022basq, pmlr-v162-sakr22a} provide some relief by clipping these outliers, yet the core issue is that outliers in feature maps are essential and cannot merely be overlooked. 
Alternatives such as clustering quantization~\cite{abernathy2022incremental} and learning-based quantization strategies~\cite{gong2019differentiable} show promise in diminishing these compression errors. 
However, such approaches' computational overhead and complexity render them impractical for high-frequency iterative training scenarios, where rapid and efficient computation is paramount. 
In addition, sparsification converts dense vectors into their sparse forms. 
A prevalent technique is the Top-$k$ sparsification~\cite{alistarh2018convergence}, which retains only a vector's most significant $k$ values. 

\paragraph{Unbiased Compression Techniques}
Unbiased random quantization~\cite{alistarh2017qsgd} distinguishes itself through random rounding procedures, marking an advancement in quantization that guarantees the operation's unbiased nature.
Similarly, there is also an unbiased variant of sparsification~\cite{wangni201I8gradient} by enlarging randomly retained values, which guarantees an unbiased sparsification process. 

In the learning paradigm of transmission gradients, such as in FL, applying unbiased compression techniques has been shown to enable DNNs to converge and realize a convergence rate the same as uncompressed gradient~\cite{alistarh2017qsgd, wangni201I8gradient}.
However, applying unbiased compression directly to feature maps in SL does not safeguard against the production of unbiased gradients.
Additionally, these compression technologies tend to induce substantial compression errors, challenging their applicability for feature map compression in SL.

\section{Feature Map Compression}
This section presents the problem formulation associated with the feature map compression within the SplitFed learning paradigm~\cite{thapa2022splitfed}. 
We further provide an in-depth analysis of gradients and convergence when applying feature map compression techniques in SL.

\subsection{Problem Formulation}
Consider a SL framework consisting of $N$ clients. 
Client $i$ has a local model $f_i^c(\cdot)$ parameterized by $\theta_i^c$, while its complementary server-side model is denoted as $f_i^s(\cdot)$, parameterized by $\theta_i^s$. 
Furthermore, client $i$ has an unbiased sample dataset $x_{i}$ and corresponding labels $y_{i}$.

During the forward propagation phase, client $i$ transmits the smashed data, $f_i^c(\theta_i^c; x_{i})$, along with the labels $y_{i}$ to the server for further computations. 
The challenge lies in the transmission of $f_i^c(\theta_i^c; x_{i})$, which is the primary bottleneck in the SL training process, primarily due to the generally slower uplink speeds than downlink speeds. 
We focus on compressing the smashed data, with the compression error at client $i$ denoted by $\epsilon_{i}$. 
The server computes the aggregate loss as follows:
\begin{align}
L(\theta_1^c, \theta_1^s, \ldots, \theta_N^c, \theta_N^s; x_1, y_1, \epsilon_1, \ldots, x_N, y_N, \epsilon_N)\nonumber \\
\coloneqq \frac{1}{N} \sum_{i=1}^{N} f_i^s(\theta_i^s; f_i^c(\theta_i^c; x_{i}) + \epsilon_i, y_i).
\end{align}

During backpropagation, the server computes the gradient $\hat{g_i^s} = \nabla_{\theta_i^s} L$, followed by the mean gradient $\hat{g^s} = \frac{1}{N} \sum_{i=1}^{N} \hat{g_i^s}$ for model parameter updates.
The gradient concerning each client's smashed data, $\nabla_{f_i^c(\theta_i^c; x_i+\epsilon_i)} L$, returns to the respective client $i$ for parameter adjustment. 
Subsequently, client $i$ calculates its parameter gradient, $\hat{g_i^c} = \nabla_{\theta_i^c} {f_i^c(\theta_i^c; x_i+\epsilon_i)} \nabla_{f_i^c(\theta_i^c; x_i+\epsilon_i)} L$. 
All clients then collaborate to aggregate these gradients through either a server-based or an all-reduce method, yielding $\hat{g^c} = \frac{1}{N} \sum_{i=1}^{N}\hat{g_i^c}$.
Both clients and the server update their respective parameters employing stochastic gradient descent with a learning rate $\eta$:
\begin{align}
    \theta_i^c \coloneqq \theta_i^c - \eta \cdot \hat{g^c} , \quad \theta_i^s \coloneqq \theta_i^s - \eta \cdot \hat{g^s}.
\end{align}

In scenarios where no compression is applied to the feature maps ($\epsilon_i = 0$), we denote the server gradient, client gradient, aggregated server gradient, and aggregated client gradient as $g_i^s$, $g_i^c$, $g^s$, and $g^c$, respectively.

Given the aggregation of parameters in each round, the objective function can be simplified to:
\begin{align}
    F(\theta^c, \theta^s) \coloneqq L(\theta^c, \theta^s, \ldots, \theta^c, \theta^s),
\end{align}
thereby rendering the SL's optimization objective as:\footnote{In this paper, Use $[v_1, v_2]$ to represent column vectors $v_1$ and $v_2$ concatenated by columns.}
\begin{align}
[\theta^c, \theta^s]^* \coloneqq \mathop{\arg\min}\limits_{\theta^c,\theta^s}F(\theta^c,\theta^s).
\end{align}

\subsection{Gradient Analysis}
In FL, communication compression techniques such as unbiased random quantization and rand-$k$ sparsification are commonly utilized to preserve the unbiased gradients at the parameter server. 
However, even when these unbiased compression methods are applied to feature maps within the SL framework, the gradients are biased.

We state this formally as follows:
\begin{theorem}[Compressed Feature Maps Bring Biased Gradients]
\label{biasedgradient}
Even unbiased compression techniques ($\mathbb{E}(\epsilon_i)=0$) are applied to feature maps in SL, the outcome inevitably leads to biased gradients:
\begin{align}
    \mathbb{E}([\hat{g^c}, \hat{g^s}]) \neq [g^c, g^s]
\end{align}
\end{theorem}

\begin{proof}
Assume that the unbiased server-side gradient is achieved, i.e., $\mathbb{E}(\hat{g^s}) = g^s$. 
Denote the original feature map of client $i$ as $z_{i}$ and its compressed version as $\hat{z_{i}}$, with the expectation $\mathbb{E}(\hat{z_i}) = z_i$. 
The server-side gradient is expressed as a function of the feature map, denoted as $h(\cdot)$, such that for the original feature map we have $g^s_i = h(z_i)$, and for the compressed feature map, $\hat{g^s_i} = h(\hat{z_i})$. 
It follows that $\mathbb{E}(\hat{g_i^s}) = \mathbb{E}(h(\hat{z_i}))$ and $h(\mathbb{E}(\hat{z_i}))=h(z_i)=g_i^s$.

If $\mathbb{E}(\hat{g_i^s}) = g_i^s$, for any server-side gradient function $h(\cdot)$ and any feature map $z$:
\begin{equation}
    \mathbb{E}(h(z)) = h(\mathbb{E}(z)).
    \label{eq:theorem_proof}
\end{equation}

Consider the widely used activation function ReLU, defined as $a(x) = \max(0, x)$, and let $z$ be a random variable that adheres to a standard normal distribution. 
We find that $\mathbb{E}(a(z)) = \int_0^{\infty}x\frac{1}{\sqrt{2\pi}}\exp^{-\frac{x^2}{2}}dx > 0$ but $a(\mathbb{E}(z)) = a(0) = 0$. 
Therefore $\mathbb{E}(a(z)) \neq a(\mathbb{E}(z))$. 
This observation contradicts Eq.~\eqref{eq:theorem_proof} because DNN typically includes many activation functions.

Thus, the assumption does not hold, which means that we can not obtain an unbiased estimate of the server-side gradient $\mathbb{E}(\hat{g^s})$ under the compression: $\mathbb{E}(\hat{g^s}) \neq g^s$.
Consequently, the expectation of the composite gradient $\mathbb{E}([\hat{g^c}, \hat{g^s}])$ does not equal the true gradients $[g^c, g^s]$.
\end{proof}

\subsection{Convergence Analysis}

In this section, we conduct a comprehensive convergence analysis of compressed feature maps in SL. 
Our analysis is based on assumptions commonly employed in distributed optimization~\cite{castiglia2022compressed, gruntkowska2023ef21, markov2023quantized}:

\begin{assumption}[Lipschitz Smoothness]
\label{ass_L}
There exists a finite positive constant $L$ such that for all $\theta^c_1, \theta^c_2$ in the client-side model and $\theta^s_1, \theta^s_2$ in the server-side model, the Lipschitz condition is satisfied as follows:
\begin{align}
\|\nabla F(\theta^c_1, \theta^s_1) - \nabla F(\theta^c_2, \theta^s_2)\| \le L \| [\theta^c_1, \theta^s_1] - [\theta^c_2, \theta^s_2] \|.
\end{align}
\end{assumption}

\begin{assumption}[Unbiased Gradients]
\label{ass_unbias}
For every mini-batch, the stochastic gradients of the client-side and server-side models are unbiased estimators of the true gradient:
\begin{align}
\mathbb{E}([g^c,g^s]) = \nabla F(\theta^c,\theta^s).
\end{align}
\end{assumption}

\begin{assumption}[Bounded Variance]
\label{ass_var}
The variances of the stochastic gradients in the client-side and server-side models are bounded by a finite positive constant $\sigma^2$:
\begin{align}
\mathbb{E}\|\nabla F(\theta^c,\theta^s)-[g^c, g^s]\|^2\le \sigma^2.
\end{align}
\end{assumption}

Given Assumptions~\ref{ass_L}-\ref{ass_var}, we establish the following lemma to bound the gradient errors in the client-side and server-side models.

\begin{lemma}[Bounded Gradient Error]\label{gradientbound}
Let \( J \) and \( H \) denote the upper bounds of the Frobenius norms of the Jacobian matrix for the client-side model and the Hessian matrix for the server-side model, respectively. 
Define $E \coloneqq \frac{1}{N}\sum_{i=1}^{N}\|\epsilon_i\|$, where $\epsilon_i$ represents the compression error for the $i$-th client. 
The gap in the server-side gradient \( g^s \) and its compressed counterpart \( \hat{g^s} \), as well as those in the client-side gradient \( g^c \) and \( \hat{g^c} \), are bounded as follows:
\begin{align}
\| \hat{g^s} - g^s \|\le HE,\quad \| \hat{g^c}-g^c\|\le HJE.
\end{align}
\end{lemma}
\begin{proof}
Denote server-side hessian function as $H^s(\cdot)$. According to the Lagrange mean value theorem, there are a set of points $\xi_i(i=1,2\dots N)$ yields: 
\begin{align}
    &\| \hat{g^s}-g^s \| 
    \\
    =
    &\| \frac{1}{N}\sum_{i=1}^{N}(\hat{g_i^s}-g_i^s)\|
    \\
    =
    &\|\frac{1}{N}\sum_{i=1}^{N}(\nabla_{\theta_i^s} f^s(\theta_i^s;\hat{f_i^c}(\theta_i^c;x_i),y_i)  \nonumber
    \\
    &-\nabla_{\theta_i^s} f^s(\theta_i^s;f_i^c(\theta_i^c;x_i),y_i))\|
    \\
    =
    &\|\frac{1}{N}\sum_{i=1}^{N}H^s(\xi_i)^\top(\hat{f_i^c}(\theta_i^c;x_{i})-f_i^c(\theta_i^c;x_{i}))\|
    \\
    =
    &\|\frac{1}{N}\sum_{i=1}^{N}H^s(\xi_i)^\top \epsilon_i\|
    \\
    \le
    &\frac{1}{N}\sum_{i=1}^{N}\|H^s(\xi_i)\|_F\|\epsilon_i\|
    \le HE.
\end{align}

Use identity matrix for approximate differentiation of compression operation, and follow the chain derivation rule:
\begin{align}
    \quad\quad\quad&\|\hat{g^c}-g^c\|
    \\
    =
    &\|\frac{1}{N}\sum_{i=1}^N\hat{g_i^c}-g^c\|
    \\
    =
    &\|\frac{1}{N}\sum_{i=1}^N\nabla_{\theta_i^c}f^s(\theta^s;f_i^c(\theta_i^c;x_i)+\epsilon_i)-g^c\|
    \\
    =
    &\|\frac{1}{N}\sum_{i=1}^N\nabla_{f_i^c(\theta_i^c;x_i)} f^s(\theta^s;f_i^c(x_i;\theta_i^c)+\epsilon_i)^\top \nonumber
    \\
    &\nabla_{\theta_i^c}f_i^c(\theta_i^c;x_i)-g^c\|
    \\
    =
    &\|\frac{1}{N}\sum_{i=1}^N(\nabla_{f_i^c(\theta_t^c;x_i)} f^s(\theta^s;f_i^c(\theta_i^c;x_i))-g^c\|
    \\
    =
    &\|\frac{1}{N}\sum_{i=1}^N\nabla_{\theta_i^c}f^s(\theta^s;f_i^c(\theta^c_i;x_i) )
    \\
    &+\frac{1}{N}\sum_{i=1}^NH^s(\xi_i)\epsilon_i^\top\nabla_{\theta_i^c}f_i^c(x_i;\theta_i^c)-g^c\|
    \\
    =
    &\|\frac{1}{N}\sum_{i=1}^N g_i^c+\frac{1}{N}\sum_{i=1}^NH^s(\xi_i)\epsilon_i^\top\nabla_{\theta_i^c}f_i^c(x_i;\theta_i^c)-g^c\|
    \\
    =
    &\|\frac{1}{N}\sum_{i=1}^NH^s(\xi_i)\epsilon_i^\top\nabla_{\theta_i^c}f_i^c(x_i;\theta_i^c)\|
    \\
    \le&\frac{1}{N}\sum_{i=1}^{N}\|H^s(\xi_i)\|_F\|\epsilon_i\|\|\nabla_{\theta_i^c}f_i^c(x_i;\theta_i^c)\|
    \le HJE. 
\end{align}

Therefore, Lemma~\ref{gradientbound} is proven.

\end{proof}

Building on Assumptions~\ref{ass_L}-\ref{ass_var} and Lemma~\ref{gradientbound}, we derive the convergence rate in SL as follows:

\begin{theorem}
[Convergence in SL] 
\label{convergenceserverclient}
Under Assumptions~\ref{ass_L}-\ref{ass_var} and Lemma~\ref{gradientbound}, use subscripts to represent the number of iterations. Consider $\varepsilon^2=(1+J^2)H^2\frac{1}{T}\sum_{t=1}^{T}E_t$ and $\gamma = \mathbb{E}(F(\theta_{1}^c,\theta_{1}^s)-F(\theta_{T+1}^c,\theta_{T+1}^s))$. By choosing the learning rate 
$
    \eta = \sqrt{\frac{\gamma}{TL(\sigma^2+\varepsilon^2)}}
$ and ensuring that the inner product between the error gradient (e.g., $[\hat{g^c},\hat{g^s}]-[g^c,g^s]$) and the full gradient (e.g., $\nabla_{[\theta^c,\theta^s]} F(\theta^c,\theta^s)$) is always non-negative, with $\eta\in(0,\frac{1}{2L}]$, SL can achieve the following convergence rate:
\begin{align}    
    \frac{1}{T}\sum_{t=1}^{T}\mathbb{E}\|\nabla F(\theta_t^c,\theta_t^s)\|^2\le 4\sqrt{\frac{\gamma L(\sigma^2+\varepsilon^2)}{T}}.
\end{align}
In other scenarios, with $\eta\in(0,\frac{1}{4L}]$, the convergence rate is given by:
\begin{align}
    \frac{1}{T}\sum_{t=1}^{T}\mathbb{E}\|\nabla F(\theta_t^c,\theta_t^s)\|^2\le 8\sqrt{\frac{\gamma L(\sigma^2+\varepsilon^2)}{T}}+2\varepsilon^2.
\end{align}
\end{theorem}
\begin{proof}
    Please refer to the Appendix A of~\cite{2024mask} for the proof of Theorem~\ref{convergenceserverclient}.
\end{proof}

\begin{remark}
Theorem~\ref{biasedgradient} indicates the effect of feature map compression on gradient bias within the SL framework. 
This phenomenon persists even when employing ostensibly unbiased compression methodologies, and such a biased gradient hurts the model's convergence. 
Complementing this, Theorem~\ref{convergenceserverclient} delineates how the compression errors directly influence the convergence rates of both the client and server-side models. 
These conclusions indicate that previous compression techniques, such as quantization and sparsification, are unsuitable for SL because they can not obtain unbiased gradients and introduce substantial errors. 
\end{remark}

\section{Mask-Encoded Sparsification}
In this section, we introduce the mask-encoded sparsification (MS) algorithm and provide theoretical and experimental evidence to demonstrate that the algorithm's error is lower than that of traditional methods.

\subsection{Algorithm Details}
Inspired by the max pooling algorithm, we believe that the top values within a feature map are critical in DNNs.
However, sparsification leads to significant compression errors, especially with a higher sparsification ratio. 
As pointed out in Theorem~\ref{convergenceserverclient}, this can negatively affect the model's convergence.
To address this, we propose a new method to weaken the values lost due to sparsification, thus reducing the impact of compression errors.

For vanilla Top-$k$ sparsification, we use a mask that matches the size of the feature map.
We set bits in the mask to $0b1$ for the top values and $0b0$ for the filtered values, storing all top values in order. 
During the decoding process, we traverse the mask and sequentially fill in the top values wherever the mask has a bit $0b1$, replacing all other values with $0b0$. 
This sparse matrix storage method costs less than the key-value pair storage method when the sparsification ratio is below $96.875\%$ (Please refer to the Appendix B of~\cite{2024mask} for more details).

In our method, we expand the mask bit width to \( b \) and assign a binary number composed of \( b \) ones to the position mask of the top values. 
This approach allows the additional bits to compensate for the values filtered by Top-$k$ sparsification.
When the activation value is non-negative (such as after applying the ReLU function), all filtered values range from 0 to the smallest top value, allowing us to divide this range into \( 2^b-1 \) segments evenly. 
For a filtered value \( t \), the corresponding mask position records the binary form of \( \lfloor t \times (2^b - 1) / Top_{\text{min}} \rfloor \), where \( Top_{\text{min}} \) is the smallest top value. If the activation values include negative numbers, an additional bit is needed in the mask to record the sign, with all values treated as non-negative.
The client transmits top values and the encoded mask to the server. 

Upon receiving this data, the server first identifies the minimum top value \( Top_{\text{min}} \). Then, it iterates through the mask. Whenever it encounters a number where all bits are $1$, the server pops the first value from the data vector and assigns it to the corresponding position. 
For other mask values, the server multiplies the mask value by \( Top_{\text{min}} / (2^b - 1) \) to reconstruct the filtered value.
For a detailed explanation of this process, refer to Algorithm~\ref{encode} and Algorithm~\ref{decode}.

\begin{algorithm}[htbp]
\caption{Encoding Stage of Mask-Encoded Sparsification}
\label{encode}
\textbf{Input}: $d$-dimensional feature map $x$\\
\textbf{Parameter}: Sparsification ratio $r$, mask bit width $b$\\
\textbf{Output}: Vector of top values $v$, encoded mask $m$

\begin{algorithmic}[1]
\STATE Conduct Top-$k$ sparsification on $x$, with $k=\lfloor(1-r) \times d\rfloor$. Record the smallest top value as $Top_{\text{min}}$ and arrange the top values in order to form vector $v$.
\STATE Initialize a $d$-dimensional mask with $b$ bits per position.
\FOR{$i = 1$ to $d$}
    \IF{$x_{i} \ge Top_{\text{min}}$}
        \STATE $m_i \gets 2^b-1$.
    \ELSE
        \STATE $m_i \gets \lfloor x_i \times (2^b - 1) / Top_{\text{min}} \rfloor$.
    \ENDIF
\ENDFOR
\STATE \textbf{return} vector $v$ and mask $m$.
\end{algorithmic}
\end{algorithm}

\begin{algorithm}[htbp]
\caption{Decoding Stage of Mask-Encoded Sparsification}
\label{decode}
\textbf{Input}: Vector of top values $v$, encoded mask $m$\\
\textbf{Parameter}: Mask bit width $b$, dimension of feature map $d$\\
\textbf{Output}: Decoded feature map $\hat{x}$

\begin{algorithmic}[1]
\STATE Initialize $\hat{x}$ with the same dimension as $m$.
\STATE Retrieve the minimum top value $Top_{\text{min}}$ from $v$.
\FOR{$i = 1$ to $d$}
    \IF{$m_{i} = 2^b - 1$}
        \STATE Pop the first value $t$ from $v$.
        \STATE $\hat{x}_i \gets t$.
    \ELSE
        \STATE $\hat{x}_i \gets m_i \times Top_{\text{min}} / (2^b - 1)$.
    \ENDIF
\ENDFOR
\STATE \textbf{return} the reconstructed feature map $\hat{x}$.
\end{algorithmic}
\end{algorithm}

In Fig.~\ref{sl}, the client generates a $2$-bit mask with the same dimensions as the feature map. Subsequently, Top-$4$ sparsification is applied to the feature map, resulting in a data vector containing the top values in their original order. 
The corresponding mask position is assigned a binary number with all $1$ bits (here is $0b11$). 
The smallest top value recorded during this process is $2.1$. 
Additionally, the client uniformly maps all filtered values within the $[0, 2.1]$ range to integers ranging from $0$ to $3$ (where $3 = 2^2 - 1$). 
These mapped values are converted into binary numbers and placed within the corresponding mask. 
Upon receiving the encoded mask and data vector, the server performs the following steps to decode the data. 
Firstly, it identifies the smallest top value of $2.1$. 
Next, it iterates through the mask. If a mask value of all $0b11$ is encountered, the server pops the first element from the data vector and fills it into the corresponding position. 
Otherwise, the mask value is mapped to the interval $[0, 2.1]$ and filled into the recovered feature map. 
By following this process, the server successfully obtains the decoded data.

MS effectively compresses the data size through this approach, reducing communication overhead.
The example illustrated in Fig.~\ref{sl} demonstrates a reduction in traffic from $64$ bytes to $20$ bytes compared to uncompressed data. 
Furthermore, it decreases the 2-norm compression error from $4.06$ to $1.02$ compared to vanilla Top-$k$ sparsification.

\subsection{Error Analysis}

We compare the error rates of our new algorithm (MS) with a couple of existing ones: Top-$k$ sparsification (SP)~\cite{alistarh2018convergence,pmlr-v202-xu23v,Zhou_2023_ICCV}, quantization (QU)~\cite{alistarh2017qsgd,pmlr-v202-dorfman23a,10228970}, and randomized Top-$k$ sparsification (RT)~\cite{rt}.

\paragraph{Error upper bound comparison}
In the RT method, sparsification is achieved by repeated sampling, with the top values having a higher probability of being selected. 
It can be proven that RT introduces more compression errors than SP. Therefore, we do not consider the compression error between RT and MS. 

Assuming that the feature map is \( x \) of length \( d \), where each value has a bit width of \( f \), the values in QU are mapped to a bit width of \( q_1 \). 
In SP, \( k_1 \) top values are retained, while in the MS approach, \( k_2 \) top values are kept, utilizing a $q_2$-bit mask. 
The compression ratios of these three algorithms, QU, SP, and MS, are equal when the conditions \( q_1d = q_2d + fk_2 \) and \( d + fk_1 = q_2d + fk_2 \) are met.

The error upper bounds for QU and SP are defined by Eq.~\eqref{qu_error}~\cite{alistarh2017qsgd} and Eq.~\eqref{sp_error}~\cite{alistarh2018convergence}, respectively. 
In our method, the error primarily arises from the compensation of values filtered out by SP. 
Consequently, the upper bound of the compression error in MS adheres to Eq.~\eqref{ms_error}. Here, \( \alpha \) represents the ratio of the 2-norm of the values filtered by SP to the $2$-norm of \( x \).
\begin{align}
    \mathbb{E}\|QU(x)-x\|^2 &\le \sqrt{d}/(2^{q_1}-1)\|x\|^2, \label{qu_error}
    \\
    \mathbb{E}\|SP(x)-x\|^2&\le(d-k_1)/d\|x\|^2, \label{sp_error}
    \\
    \mathbb{E}\|MS(x)-x\|^2 &\le \alpha\sqrt{d-k_2}/(2^{q_2}-1)\|x\|^2. \label{ms_error}
\end{align}

\begin{theorem}
[Conditions superior to QU and SP]
\label{better_than_qu_sp}
With the increase of $q_2$, the upper bound error of MS is less than SP.
When $\alpha\in(0, 1/2)$ and $k_2/d \rightarrow 0$, the upper bound error of MS is less than QU. 
\end{theorem}
\begin{proof}
        Please refer to the Appendix C of~\cite{2024mask} for the proof of Theorem~\ref{better_than_qu_sp}.
\end{proof}

\begin{figure*}[htbp]
    \centering
    \includegraphics[width=0.7\textwidth]{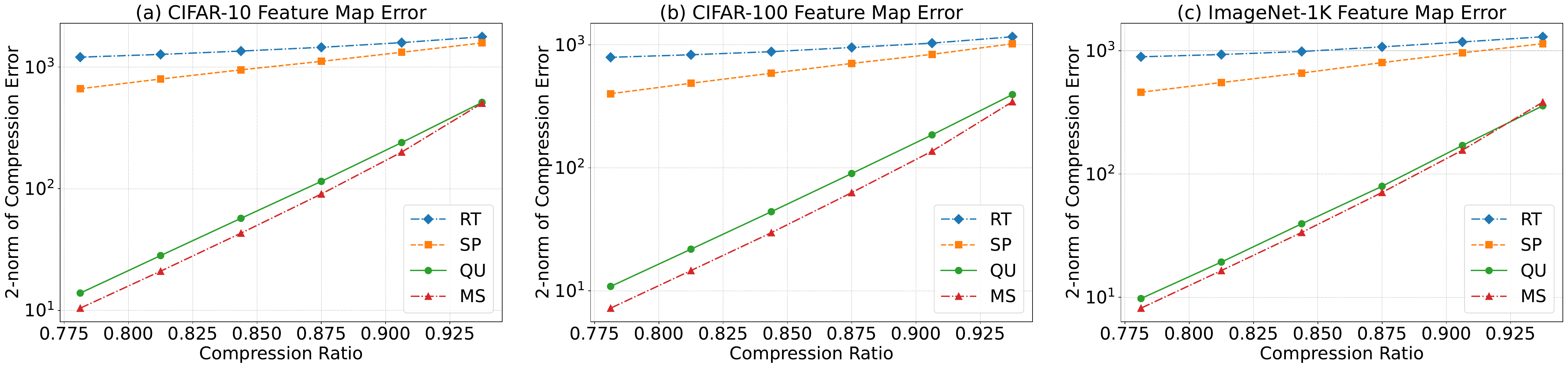}
    \caption{The average 2-norm compression error of feature maps for QU, SP, MS, and RT on different datasets as the compression rate increases. Figures (a), (b), and (c) show the comparison of CIFAR-10, CIFAR-100, and ImageNet-1K, respectively.}
    \label{error_compare}
\end{figure*}

\paragraph{Error comparison experiment}

We examined the compression errors of feature maps on CIFAR-10, CIFAR-100, and ImageNet-1K, employing SP, QU, MS, and RT algorithms at identical compression levels. The comparison is illustrated in Fig.~\ref{error_compare}. 

Notably, the compression error of MS is always lower than QU, SP, and RT. 
MS compensates for the values filtered out by SP through the mask, resulting in an error lower than SP.
Additionally, MS preserves top values, thus avoiding the sensitivity of quantization to outliers, and the error is also lower than quantization. The experiment also proves that the compression error of RT is consistently higher than that of SP at the same compression rate.

In addition, the experiment also verified the Theorem~\ref{better_than_qu_sp} that as $q_2$ increases (compression rate decreases), the gap between the compression error of MS and SP gradually increases. 
When the sparsification rate increases (compression rate increases), the compression error of MS is lower than that of QU.

\begin{remark}
Building on Theorems~\ref{better_than_qu_sp}, it is observed that MS achieves a lower upper bound on the compression error compared to QU and SP. 
This performance improvement is noted under conditions of increased sparsification ratio and expanded mask bit width, which are relatively mild constraints. 
Furthermore, Fig.~\ref{error_compare} illustrates that MS outperforms QU, SP and RT regarding compression error.
The theoretical insights and empirical results consistently underscore the superiority of MS over QU, SP, and RT in terms of compression error. This advantage facilitates improved convergence in SL.
\end{remark}

\subsection{Time Complexity}
During the encoding phase, SP exhibits a time complexity of $\mathcal{O}(d \log(k))$, attributed to the Top-$k$ operation. 
Here, $d$ denotes the feature map length, and $k$ represents the retained top values.
Both MS and RT are improved versions based on SP. 
MS introduces an additional operation of mask encoding, with a time complexity of $\mathcal{O}(d)$, so the total time complexity of MS is $\mathcal{O}(d\log(k) + d)$.
However, RT introduces $k$ sampling operations, where the time complexity of each sampling is $\mathcal{O}(d)$, and the total time complexity of RT is $\mathcal{O}(d\log(k) + dk)$. Consequently, MS provides powerful performance without increasing the complexity of the time order. In contrast, RT increased the time complexity of SP from $\mathcal{O}(d\log(k))$ to $\mathcal{O}(dk)$, which affected the training efficiency of SL.

Some compression methods, such as clustering quantization~\cite{abernathy2022incremental} and learning-based quantization~\cite{gong2019differentiable}, may achieve lower compression errors than MS but at the expense of a higher time complexity during both encoding and decoding phases. 
This drawback renders them less suitable for SL, where iterative training is frequent and time efficiency is paramount.

\begin{figure*}[htbp]
    \centering
    \includegraphics[width=0.73\textwidth]{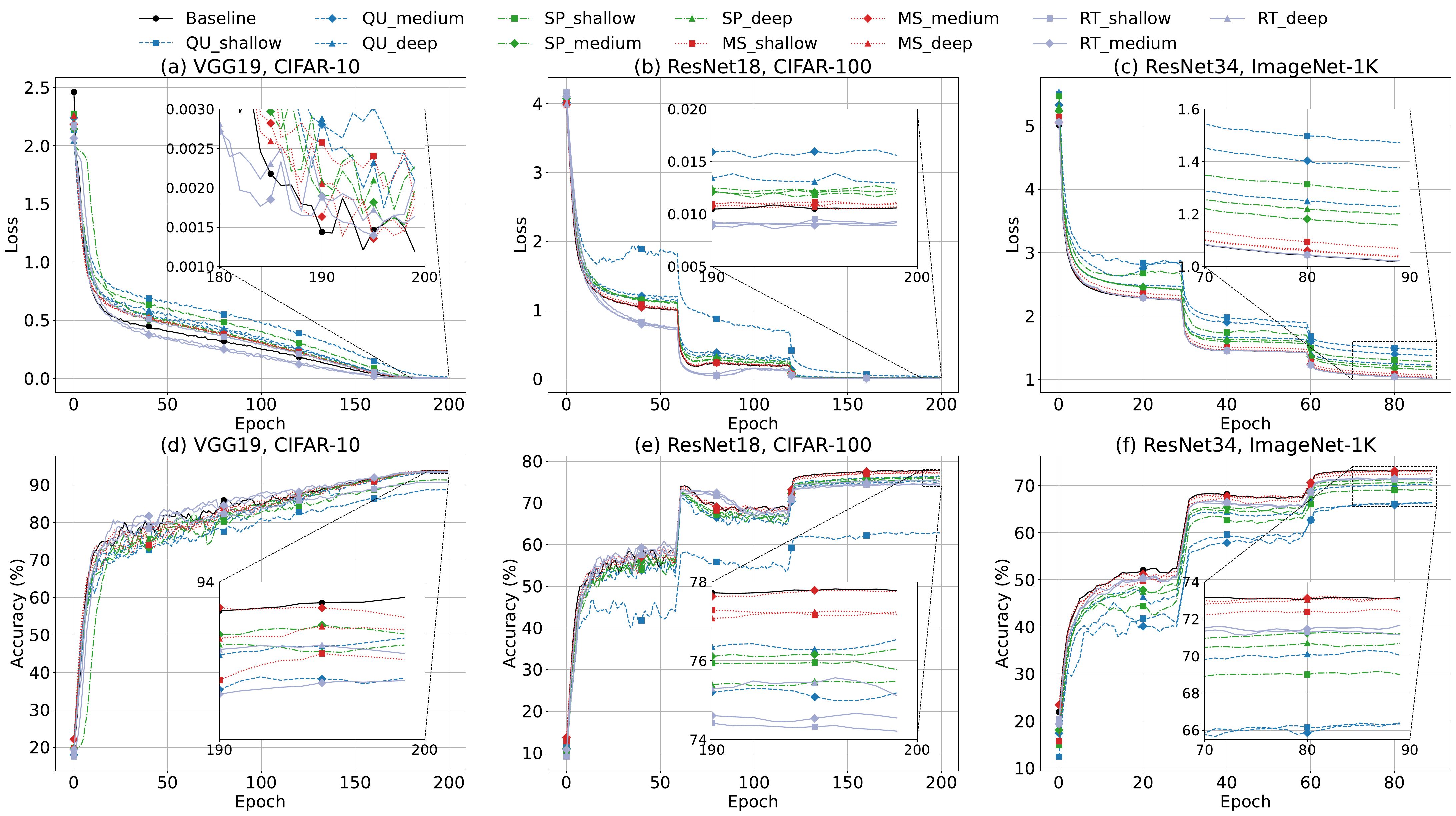}
    \caption{The experiment results of convergence and generalization. The first and second rows represent the changes in training loss and testing accuracy, respectively. The first, second, and third columns represent the results of CIFAR-10, CIFAR-100, and ImageNet-1K, respectively. The solid black line indicates the baseline trend, the blue dotted line depicts the changes in QU, the green stippled line shows the changes in SP, and the red dotted line represents the changes in MS.  Additionally, square, diamond, and triangular markers correspondingly denote shallow, middle, and deep cutlayers.}
    \label{exp_result}
    \vspace{5pt}
\end{figure*}

\section{Experiments}

\subsection{Experimental Setup}
\paragraph{Devices}
To simulate a realistic edge environment, we use HUAWEI Atlas DK200 as the client, connected to an NVIDIA GeForce RTX 3090 server via a 100Mbps wireless network.

\paragraph{Benchmarks}
We train three models on three datasets: VGG19~\cite{simonyan2014very} on CIFAR-10~\cite{cifar}, ResNet18~\cite{he2016deep} on CIFAR-100~\cite{cifar}, and ResNet34~\cite{he2016deep} on ImageNet-1K~\cite{russakovsky2015imagenet}. 
We evenly cut the three datasets into ten parts and put them into ten clients.
For these three models, we apply three cutting strategies, which are shallow, medium, and deep cuts. Specifically, for VGG19, we cut the model after layers $2$, $8$, and $15$, respectively; for ResNet18, we cut the model after layers $2$, $9$, and $13$, respectively; and for ResNet34, we cut the model after layers $2$, $15$, and $27$, respectively. 

\paragraph{Compression Policies}
When using MS compression, the sparsification ratio is $99\%$, the bit width is $2$, the compression rate is $92.75\%$.
We compare $95.875\%$ SP and $95.875\%$ RT at our MS's compression level. 
Due to the unsmooth compression ratio of quantization, we use $3$-bit quantization with a compression rate of $90.625\%$. 
We also recorded the uncompressed version as the baseline.

\paragraph{Metrics}
We compare the convergence and generalization of the models by analyzing the average training loss and test accuracy curves as the epoch increases. 
In addition, we also compare the traffic multiples saved by all compression strategies to achieve baseline accuracy.

\subsection{Experimental Results}

\paragraph{Convergence and Generalization}
The convergence and generalization performance for VGG19, ResNet18, and ResNet34 are illustrated in Fig.~\ref{exp_result}, respectively. 
In these figures, the solid black line indicates the baseline trend, the blue dotted line depicts the changes in QU, the green stippled line shows the changes in SP, and the red dotted line represents the changes in MS. 
Additionally, square, diamond, and triangular markers correspondingly denote shallow, middle, and deep cutlayers.

As depicted in Fig.~\ref{exp_result}, it is observed that in most situations, QU, SP, and struggle to converge to the baseline performance level and fail to achieve baseline accuracy (only medium and deep cuts in VGG19, QU, and SP can reach the baseline level).
RT can converge better than MS in certain situations (ResNet18 and ResNet34), but its generalization is still poor.
These phenomenons underscore the limitations of QU, SP, and RT in the context of feature map compression and are due to the vast compression errors induced by QU, SP, and RT.
On the contrary, MS can match the baseline in terms of convergence and generalization in most situations (only the situation of shallow cut in ResNet34 is slightly lower than the baseline level, and its accuracy is approximately $0.5\%$ less) because MS ensures the integrity of the top value and achieves lower compression error. 
These results highlight the superiority of the MS algorithm in compressing feature maps.

We also discover an interesting phenomenon: all compression algorithms' convergence and generalization abilities tend to improve as the cutting layer deepens. 
This phenomenon suggests that the shallow layers of neural networks are more sensitive to errors, whereas the deeper layers exhibit greater error tolerance. 

\begin{table}[htbp]
    \centering
    \caption{Comparison of communication traffic savings achieved by various compression algorithms compared to non-compression baseline, with "INF" indicating failure to achieve baseline accuracy level.} 
    \label{traffictab}
    \begin{tabular}{lllllll}
        \toprule
        \multirow{2}{*}{Datasets} &
        \multirow{2}{*}{Cutlayers} & \multicolumn{4}{c}{Traffic Saving Multiplier} \\
        \cmidrule(lr){3-6}
        & & QU & SP & MS & RT \\
        \midrule
        \multirow{3}{*}{CIFAR-10} &
        Shallow & INF & INF & \textbf{13.11$\times$} & INF\\
        & Medium & 9.92$\times$ & 13.56$\times$ & 13.56$\times$ & \textbf{13.71$\times$} \\
        & Deep & 10.37$\times$ & 13.56$\times$ & 13.63$\times$ & \textbf{13.71$\times$} \\
        \midrule
        \multirow{3}{*}{CIFAR-100} &
        Shallow & INF & INF & \textbf{12.53$\times$} & INF\\
        & Medium & INF & INF & \textbf{13.50$\times$} & INF\\
        & Deep & INF & INF & \textbf{12.29$\times$} & INF\\
        \midrule
        \multirow{3}{*}{ImageNet-1K} &
        Shallow & INF & INF & INF & INF \\
        & Medium & INF & INF & \textbf{13.20$\times$} & INF \\
        & Deep & INF & INF & \textbf{12.32$\times$} & INF\\
        \bottomrule
    \end{tabular}
\end{table}

\paragraph{Communication Efficiency}
We record the communication traffic needed to reach the baseline test accuracy in all experiments, as shown in Table~\ref{traffictab}. 
Bold values for communication traffic indicate the minor overhead each compression policy requires to meet the target test accuracy. 
We mark cases where an experiment does not reach the target accuracy under the communication compression policy with "INF".

Overall, QU, SP and RT often fail to achieve baseline accuracy and can only converge on less complex datasets such as CIFAR-10. This is because they bring higher compression errors.
MS consistently meets the baseline accuracy in nearly all tests. Compared to the other policies, MS typically requires minor communication traffic, except in the shallow cutlayer of ResNet34, where it achieved about $72.6\%$ accuracy ($0.5\%$ lower than baseline). 
These results highlight MS's efficiency in ensuring convergence and significantly reducing communication overhead.

\section{Conclusion}
In this work, we show that compressed feature maps introduce biased gradients in SL and provide a comprehensive convergence analysis in the face of feature map compression. 
Our findings confirm the critical role of compression error in the convergence process, revealing the inherent limitations of existing compression methodologies within SL frameworks. 
To address these shortcomings, we propose the MS algorithm, which employs a narrow bit width mask to mitigate sparsification errors, a strategy that keeps accuracy without adding algorithmic complexity. Our theoretical and empirical evidence shows that the MS algorithm outperforms conventional techniques, effectively diminishing compression errors while maintaining the same compression level. 
Extensive experiments on various DNN models and datasets also demonstrate the effectiveness and efficiency of MS.

\begin{ack}
This work was partially supported by the National Natural Science Foundation of China (No. 62102131, 62306104, and 62180005), the Natural Science Foundation of Jiangsu Province of China (No. BK20210361, BK20230949, and BK20220973), the Provincial Key Research and Development Program of Hubei (No. 2023BAB065), China Postdoctoral Science Foundation (2023TQ0104), Jiangsu Excellent Postdoctoral Program (2023ZB140), and the Fundamental Research Funds for the Central Universities (No. B240201064).

\end{ack}

\bibliography{references}

\begin{thebibliography}{45}
\providecommand{\natexlab}[1]{#1}
\providecommand{\url}[1]{\texttt{#1}}
\expandafter\ifx\csname urlstyle\endcsname\relax
  \providecommand{\doi}[1]{doi: #1}\else
  \providecommand{\doi}{doi: \begingroup \urlstyle{rm}\Url}\fi

\bibitem[Abernathy and Celebi(2022)]{abernathy2022incremental}
A.~Abernathy and M.~E. Celebi.
\newblock The incremental online k-means clustering algorithm and its application to color quantization.
\newblock \emph{Expert Systems with Applications}, 207, 2022.
\newblock \doi{10.1016/j.eswa.2022.117927}.

\bibitem[Alistarh et~al.(2017)Alistarh, Grubic, Li, Tomioka, and Vojnovic]{alistarh2017qsgd}
D.~Alistarh, D.~Grubic, J.~Li, R.~Tomioka, and M.~Vojnovic.
\newblock {QSGD}: Communication-efficient {SGD} via gradient quantization and encoding.
\newblock In \emph{Advances in Neural Information Processing Systems (NeurIPS 2017)}, pages 1707--1718, 2017.

\bibitem[Alistarh et~al.(2018)Alistarh, Hoefler, Johansson, Konstantinov, Khirirat, and Renggli]{alistarh2018convergence}
D.~Alistarh, T.~Hoefler, M.~Johansson, N.~Konstantinov, S.~Khirirat, and C.~Renggli.
\newblock The convergence of sparsified gradient methods.
\newblock In \emph{Advances in Neural Information Processing Systems (NeurIPS 2018)}, pages 5977--5987, 2018.

\bibitem[Beznosikov et~al.(2023)Beznosikov, Horv{\'a}th, Richt{\'a}rik, and Safaryan]{beznosikov2023biased}
A.~Beznosikov, S.~Horv{\'a}th, P.~Richt{\'a}rik, and M.~Safaryan.
\newblock On biased compression for distributed learning.
\newblock \emph{Journal of Machine Learning Research}, 24\penalty0 (276):\penalty0 1--50, 2023.

\bibitem[Castiglia et~al.(2022)Castiglia, Das, Wang, and Patterson]{castiglia2022compressed}
T.~J. Castiglia, A.~Das, S.~Wang, and S.~Patterson.
\newblock Compressed-{VFL}: Communication-efficient learning with vertically partitioned data.
\newblock In \emph{International Conference on Machine Learning (ICML 2022)}, pages 2738--2766, 2022.

\bibitem[Cheng et~al.(2023)Cheng, Xia, Liwang, Fan, Sun, Wang, and Huang]{10274134}
Z.~Cheng, X.~Xia, M.~Liwang, X.~Fan, Y.~Sun, X.~Wang, and L.~Huang.
\newblock {CHEESE}: Distributed clustering-based hybrid federated split learning over edge networks.
\newblock \emph{IEEE Transactions on Parallel and Distributed Systems}, 34\penalty0 (12):\penalty0 3174--3191, 2023.
\newblock \doi{10.1109/TPDS.2023.3322755}.

\bibitem[Condat et~al.(2022)Condat, Yi, and Richtarik]{NEURIPS2022_6fb9ea51}
L.~Condat, K.~Yi, and P.~Richtarik.
\newblock Ef-bv: A unified theory of error feedback and variance reduction mechanisms for biased and unbiased compression in distributed optimization.
\newblock In \emph{Advances in Neural Information Processing Systems (NeurIPS 2022)}, pages 17501--17514, 2022.

\bibitem[Dorfman et~al.(2023)Dorfman, Vargaftik, Ben-Itzhak, and Levy]{pmlr-v202-dorfman23a}
R.~Dorfman, S.~Vargaftik, Y.~Ben-Itzhak, and K.~Y. Levy.
\newblock {D}o{C}o{FL}: Downlink compression for cross-device federated learning.
\newblock In \emph{International Conference on Machine Learning (ICML 2023)}, pages 8356--8388, 2023.

\bibitem[Dosovitskiy et~al.(2021)Dosovitskiy, Beyer, Kolesnikov, Weissenborn, Zhai, Unterthiner, Dehghani, Minderer, Heigold, Gelly, et~al.]{dosovitskiy2021an}
A.~Dosovitskiy, L.~Beyer, A.~Kolesnikov, D.~Weissenborn, X.~Zhai, T.~Unterthiner, M.~Dehghani, M.~Minderer, G.~Heigold, S.~Gelly, et~al.
\newblock An image is worth 16x16 words: Transformers for image recognition at scale.
\newblock In \emph{International Conference on Learning Representations (ICLR 2021)}, 2021.

\bibitem[Gong et~al.(2019)Gong, Liu, Jiang, Li, Hu, Lin, Yu, and Yan]{gong2019differentiable}
R.~Gong, X.~Liu, S.~Jiang, T.~Li, P.~Hu, J.~Lin, F.~Yu, and J.~Yan.
\newblock Differentiable soft quantization: Bridging full-precision and low-bit neural networks.
\newblock In \emph{IEEE/CVF international conference on computer vision (ICCV 2019)}, pages 4852--4861, 2019.

\bibitem[Gruntkowska et~al.(2023)Gruntkowska, Tyurin, and Richt{\'a}rik]{gruntkowska2023ef21}
K.~Gruntkowska, A.~Tyurin, and P.~Richt{\'a}rik.
\newblock Ef21-p and friends: Improved theoretical communication complexity for distributed optimization with bidirectional compression.
\newblock In \emph{International Conference on Machine Learning (ICML 2023)}, pages 11761--11807, 2023.

\bibitem[Gupta and Raskar(2018)]{gupta2018distributed}
O.~Gupta and R.~Raskar.
\newblock Distributed learning of deep neural network over multiple agents.
\newblock \emph{Journal of Network and Computer Applications}, 116:\penalty0 1--8, 2018.
\newblock \doi{10.1016/j.jnca.2018.05.003}.

\bibitem[Han et~al.(2023{\natexlab{a}})Han, Kim, Choi, Brinton, and Moon]{10229027}
D.-J. Han, D.-Y. Kim, M.~Choi, C.~G. Brinton, and J.~Moon.
\newblock {SplitGP}: Achieving both generalization and personalization in federated learning.
\newblock In \emph{IEEE Conference on Computer Communications (INFOCOM 2023)}, 2023{\natexlab{a}}.
\newblock \doi{10.1109/INFOCOM53939.2023.10229027}.

\bibitem[Han et~al.(2023{\natexlab{b}})Han, Kim, Choi, Nickel, Moon, Chiang, and Brinton]{10314792}
D.-J. Han, D.-Y. Kim, M.~Choi, D.~Nickel, J.~Moon, M.~Chiang, and C.~G. Brinton.
\newblock Federated split learning with joint personalization-generalization for inference- stage optimization in wireless edge networks.
\newblock \emph{IEEE Transactions on Mobile Computing}, pages 1--17, 2023{\natexlab{b}}.
\newblock \doi{10.1109/TMC.2023.3331690}.

\bibitem[He et~al.(2016)He, Zhang, Ren, and Sun]{he2016deep}
K.~He, X.~Zhang, S.~Ren, and J.~Sun.
\newblock Deep residual learning for image recognition.
\newblock In \emph{IEEE/CVF conference on computer vision and pattern recognition (CVPR 2016)}, pages 770--778, 2016.

\bibitem[Kang et~al.(2017)Kang, Hauswald, Gao, Rovinski, Mudge, Mars, and Tang]{kang2017neurosurgeon}
Y.~Kang, J.~Hauswald, C.~Gao, A.~Rovinski, T.~Mudge, J.~Mars, and L.~Tang.
\newblock Neurosurgeon: Collaborative intelligence between the cloud and mobile edge.
\newblock \emph{ACM SIGARCH Computer Architecture News}, 45\penalty0 (1):\penalty0 615--629, 2017.
\newblock \doi{10.1145/3093337.3037698}.

\bibitem[Kim et~al.(2022)Kim, Park, and Yoo]{kim2022basq}
H.-B. Kim, E.~Park, and S.~Yoo.
\newblock Basq: Branch-wise activation-clipping search quantization for sub-4-bit neural networks.
\newblock In \emph{European Conference on Computer Vision (ECCV 2022)}, pages 17--33. Springer, 2022.
\newblock \doi{10.1007/978-3-031-19775-8_2}.

\bibitem[Krause and Stolzenburg(2023)]{krause2023commonsense}
S.~Krause and F.~Stolzenburg.
\newblock Commonsense reasoning and explainable artificial intelligence using large language models.
\newblock In \emph{European Conference on Artificial Intelligence (ECAI 2023))}, pages 302--319, 2023.
\newblock \doi{10.1007/978-3-031-50396-2_17}.

\bibitem[Krizhevsky et~al.(2009)Krizhevsky, Hinton, et~al.]{cifar}
A.~Krizhevsky, G.~Hinton, et~al.
\newblock Learning multiple layers of features from tiny images.
\newblock \url{https://www.cs.toronto.edu/~kriz/learning-features-2009-TR.pdf}, 2009.

\bibitem[Kr{\"u}ger and Gref(2023)]{kruger2023performance}
T.~Kr{\"u}ger and M.~Gref.
\newblock Performance of large language models in a computer science degree program.
\newblock In \emph{European Conference on Artificial Intelligence (ECAI 2023))}, pages 409--424, 2023.
\newblock \doi{10.1007/978-3-031-50485-3_40}.

\bibitem[Li and Li(2023)]{pmlr-v202-li23o}
X.~Li and P.~Li.
\newblock Analysis of error feedback in federated non-convex optimization with biased compression: Fast convergence and partial participation.
\newblock In \emph{International Conference on Machine Learning (ICML 2023)}, pages 19638--19688, 2023.

\bibitem[Li et~al.(2023)Li, Song, and Yang]{pmlr-v202-li23z}
X.~Li, Z.~Song, and J.~Yang.
\newblock Federated adversarial learning: A framework with convergence analysis.
\newblock In \emph{International Conference on Machine Learning (ICML 2023)}, pages 19932--19959, 2023.

\bibitem[Liu et~al.(2023)Liu, He, and Cao]{10228970}
H.~Liu, F.~He, and G.~Cao.
\newblock Communication-efficient federated learning for heterogeneous edge devices based on adaptive gradient quantization.
\newblock In \emph{IEEE Conference on Computer Communications (INFOCOM 2023)}, pages 1--10, 2023.
\newblock \doi{10.1109/INFOCOM53939.2023.10228970}.

\bibitem[Markov et~al.(2023)Markov, Vladu, Guo, and Alistarh]{markov2023quantized}
I.~Markov, A.~Vladu, Q.~Guo, and D.~Alistarh.
\newblock Quantized distributed training of large models with convergence guarantees.
\newblock In \emph{International Conference on Machine Learning (ICML 2023)}, 2023.

\bibitem[McMahan et~al.(2017)McMahan, Moore, Ramage, Hampson, and y~Arcas]{mcmahan2017communication}
B.~McMahan, E.~Moore, D.~Ramage, S.~Hampson, and B.~A. y~Arcas.
\newblock Communication-efficient learning of deep networks from decentralized data.
\newblock In \emph{Artificial intelligence and statistics (AISTATS 2017)}, pages 1273--1282, 2017.

\bibitem[Ouyang et~al.(2022)Ouyang, Wu, Jiang, Almeida, Wainwright, Mishkin, Zhang, Agarwal, Slama, Ray, et~al.]{ouyang2022training}
L.~Ouyang, J.~Wu, X.~Jiang, D.~Almeida, C.~Wainwright, P.~Mishkin, C.~Zhang, S.~Agarwal, K.~Slama, A.~Ray, et~al.
\newblock Training language models to follow instructions with human feedback.
\newblock In \emph{Advances in Neural Information Processing Systems (NeurIPS 2022)}, pages 27730--27744, 2022.

\bibitem[Qian et~al.(2021)Qian, Richtarik, and Zhang]{NEURIPS2021_ff1ced30}
X.~Qian, P.~Richtarik, and T.~Zhang.
\newblock Error compensated distributed {SGD} can be accelerated.
\newblock In \emph{Advances in Neural Information Processing Systems (NeurIPS 2021)}, pages 30401--30413, 2021.

\bibitem[Russakovsky et~al.(2015)Russakovsky, Deng, Su, Krause, Satheesh, Ma, Huang, Karpathy, Khosla, Bernstein, et~al.]{russakovsky2015imagenet}
O.~Russakovsky, J.~Deng, H.~Su, J.~Krause, S.~Satheesh, S.~Ma, Z.~Huang, A.~Karpathy, A.~Khosla, M.~Bernstein, et~al.
\newblock Imagenet large scale visual recognition challenge.
\newblock \emph{International journal of computer vision}, 115:\penalty0 211--252, 2015.
\newblock \doi{10.1007/s11263-015-0816-y}.

\bibitem[Sakr et~al.(2022)Sakr, Dai, Venkatesan, Zimmer, Dally, and Khailany]{pmlr-v162-sakr22a}
C.~Sakr, S.~Dai, R.~Venkatesan, B.~Zimmer, W.~Dally, and B.~Khailany.
\newblock Optimal clipping and magnitude-aware differentiation for improved quantization-aware training.
\newblock In \emph{International Conference on Machine Learning (ICML 2022)}, pages 19123--19138, 2022.

\bibitem[Simonyan and Zisserman(2015)]{simonyan2014very}
K.~Simonyan and A.~Zisserman.
\newblock Very deep convolutional networks for large-scale image recognition.
\newblock In \emph{International Conference on Learning Representations (ICLR 2015)}, 2015.

\bibitem[Song et~al.(2023)Song, Wang, Yu, and Zhang]{pmlr-v202-song23h}
Z.~Song, Y.~Wang, Z.~Yu, and L.~Zhang.
\newblock Sketching for first order method: Efficient algorithm for low-bandwidth channel and vulnerability.
\newblock In \emph{International Conference on Machine Learning (ICML 2023)}, pages 32365--32417, 2023.

\bibitem[Stephanie et~al.(2023)Stephanie, Khalil, and Atiquzzaman]{10234566}
V.~Stephanie, I.~Khalil, and M.~Atiquzzaman.
\newblock Digital twin enabled asynchronous splitfed learning in e-healthcare systems.
\newblock \emph{IEEE Journal on Selected Areas in Communications}, 41\penalty0 (11):\penalty0 3650--3661, 2023.
\newblock \doi{10.1109/JSAC.2023.3310103}.

\bibitem[Stich(2020)]{stich2020analysis}
A.~A. S.~U. Stich.
\newblock Analysis of {SGD} with biased gradient estimators.
\newblock \emph{arXiv preprint arXiv:2008.00051}, 2020.
\newblock \doi{10.48550/arXiv.2008.00051}.

\bibitem[Thapa et~al.(2022)Thapa, Arachchige, Camtepe, and Sun]{thapa2022splitfed}
C.~Thapa, P.~C.~M. Arachchige, S.~Camtepe, and L.~Sun.
\newblock {SplitFed}: When federated learning meets split learning.
\newblock In \emph{Proceedings of the AAAI Conference on Artificial Intelligence (AAAI 2022)}, pages 8485--8493, 2022.
\newblock \doi{10.1609/aaai.v36i8.20825}.

\bibitem[Vepakomma et~al.(2018)Vepakomma, Gupta, Swedish, and Raskar]{vepakomma2018split}
P.~Vepakomma, O.~Gupta, T.~Swedish, and R.~Raskar.
\newblock Split learning for health: Distributed deep learning without sharing raw patient data.
\newblock \emph{arXiv preprint arXiv:1812.00564}, 2018.
\newblock \doi{10.48550/arXiv.1812.00564}.

\bibitem[Wang et~al.(2023{\natexlab{a}})Wang, Fu, Du, Gao, Huang, Liu, Chandak, Liu, Van~Katwyk, Deac, et~al.]{wang2023scientific}
H.~Wang, T.~Fu, Y.~Du, W.~Gao, K.~Huang, Z.~Liu, P.~Chandak, S.~Liu, P.~Van~Katwyk, A.~Deac, et~al.
\newblock Scientific discovery in the age of artificial intelligence.
\newblock \emph{Nature}, 620\penalty0 (7972):\penalty0 47--60, 2023{\natexlab{a}}.
\newblock \doi{10.1038/s41586-023-06221-2}.

\bibitem[Wang et~al.(2023{\natexlab{b}})Wang, Lu, Yuan, Chen, Liang, De~Sa, Re, and Zhang]{pmlr-v202-wang23t}
J.~Wang, Y.~Lu, B.~Yuan, B.~Chen, P.~Liang, C.~De~Sa, C.~Re, and C.~Zhang.
\newblock {C}ocktail{SGD}: Fine-tuning foundation models over 500{M}bps networks.
\newblock In \emph{International Conference on Machine Learning (ICML 2023)}, pages 36058--36076, 2023{\natexlab{b}}.

\bibitem[Wangni et~al.(2018)Wangni, Wang, Liu, and Zhang]{wangni201I8gradient}
J.~Wangni, J.~Wang, J.~Liu, and T.~Zhang.
\newblock Gradient sparsification for communication-efficient distributed optimization.
\newblock In \emph{Advances in Neural Information Processing Systems (NeurIPS 2018)}, pages 1306--1316, 2018.

\bibitem[Xu et~al.(2023)Xu, Zhang, Fei, Wu, Xie, Huang, Xie, Elhoseiny, and Kalnis]{pmlr-v202-xu23v}
H.~Xu, W.~Zhang, J.~Fei, Y.~Wu, T.~Xie, J.~Huang, Y.~Xie, M.~Elhoseiny, and P.~Kalnis.
\newblock {SLAMB}: Accelerated large batch training with sparse communication.
\newblock In \emph{International Conference on Machine Learning (ICML 2023)}, pages 38801--38825, 2023.

\bibitem[Yin et~al.(2023)Yin, Chen, and Tao]{10129922}
B.~Yin, Z.~Chen, and M.~Tao.
\newblock Predictive gan-powered multi-objective optimization for hybrid federated split learning.
\newblock \emph{IEEE Transactions on Communications}, 71\penalty0 (8):\penalty0 4544--4560, 2023.
\newblock \doi{10.1109/TCOMM.2023.3277878}.

\bibitem[Zheng et~al.(2023)Zheng, Chen, Lyu, and Yao]{rt}
F.~Zheng, C.~Chen, L.~Lyu, and B.~Yao.
\newblock Reducing communication for split learning by randomized top-k sparsification.
\newblock In \emph{International Joint Conference on Artificial Intelligence (IJCAI 2023)}, pages 4665--4673, 2023.
\newblock \doi{10.24963/ijcai.2023/519}.

\bibitem[Zhou et~al.(2022{\natexlab{a}})Zhou, Guo, Liu, Zhang, Zhang, Guo, Xu, Liu, and Qu]{zhou2022hierarchical}
Q.~Zhou, S.~Guo, Y.~Liu, J.~Zhang, J.~Zhang, T.~Guo, Z.~Xu, X.~Liu, and Z.~Qu.
\newblock Hierarchical channel-spatial encoding for communication-efficient collaborative learning.
\newblock In \emph{Advances in Neural Information Processing Systems (NeurIPS 2022)}, pages 5788--5801, 2022{\natexlab{a}}.

\bibitem[Zhou et~al.(2022{\natexlab{b}})Zhou, Qu, Zhao, Tang, and Ye]{10.1145/3538641.3561500}
W.~Zhou, Z.~Qu, Y.~Zhao, B.~Tang, and B.~Ye.
\newblock An efficient split learning framework for recurrent neural network in mobile edge environment.
\newblock In \emph{Proceedings of the Conference on Research in Adaptive and Convergent Systems (RACS 2022)}, page 131–138, 2022{\natexlab{b}}.
\newblock \doi{10.1145/3538641.3561500}.

\bibitem[Zhou et~al.(2024)Zhou, Qu, Lyu, Cai, and Ye]{2024mask}
W.~Zhou, Z.~Qu, S.-H. Lyu, M.~Cai, and B.~Ye.
\newblock Mask-encoded sparsification: Mitigating biased gradients in communication-efficient split learning.
\newblock \emph{arXiv preprint arXiv:2408.13787}, 2024.
\newblock \doi{10.48550/arXiv.2408.13787}.

\bibitem[Zhou et~al.(2023)Zhou, Shi, Li, Sun, Ye, and Lv]{Zhou_2023_ICCV}
Y.~Zhou, M.~Shi, Y.~Li, Y.~Sun, Q.~Ye, and J.~Lv.
\newblock Communication-efficient federated learning with single-step synthetic features compressor for faster convergence.
\newblock In \emph{Proceedings of the IEEE/CVF International Conference on Computer Vision (ICCV 2023)}, pages 5031--5040, 2023.

\end{thebibliography}

\clearpage
\onecolumn
\appendix

\section{Proof of Theorem~\ref{convergenceserverclient}}
\label{appendix:convergence}
\begin{proof}
We denote $t$ and $t+1$ (the subscript) as the iteration round. From Assumption \ref{ass_L}, in the $t$-th iteration, we can derive:
\begin{align}
    \mathbb{E}(F(\theta_{t+1}^c,\theta_{t+1}^s)-F(\theta_t^c,\theta_t^s))
    &\le
    \mathbb{E}\langle \nabla F(\theta_t^c,\theta_t^s),[\theta_{t+1}^c,\theta_{t+1}^s]-[\theta_t^c,\theta_t^s]\rangle 
    +
    \frac{L}{2}\mathbb{E}\|[\theta_{t+1}^c,\theta_{t+1}^s]-[\theta_t^c,\theta_t^s]\|^2
    \\
    &=
    \label{441}
    -\eta\mathbb{E}\langle \nabla F(\theta_t^c,\theta_t^s), [\hat{g_t^c},\hat{g_t^s}]\rangle
    +
    \frac{L\eta^2}{2}\mathbb{E}\|[\hat{g_t^c},\hat{g_t^s}]\|^2
    \\
    &=
    -\eta\mathbb{E}\langle \nabla F(\theta_t^c,\theta_t^s), [g_t^c,g_t^s]+[\hat{g_t^c},\hat{g_t^s}]-[g_t^c,g_t^s]\rangle
    +
    \frac{L\eta^2}{2}\mathbb{E}\|[g_t^c,g_t^s]+[\hat{g_t^c},\hat{g_t^s}]-[g_t^c,g_t^s]\|^2
    \\
    &\le
    -\eta\mathbb{E}\langle \nabla F(\theta_t^c,\theta_t^s), [g_t^c,g_t^s]\rangle
    -
    \eta\mathbb{E}\langle \nabla F(\theta_t^c,\theta_t^s), [\hat{g_t^c},\hat{g_t^s}]-[g_t^c,g_t^s]\rangle
    \nonumber
    \\
    &+
    L\eta^2\mathbb{E}\|[g_t^c,g_t^s]\|^2
    +
    L\eta^2\mathbb{E}\|[\hat{g_t^c},\hat{g_t^s}]-[g_t^c,g_t^s]\|^2
    \label{442}
    \\
    &=
    (L\eta^2-\eta)\mathbb{E}\|\nabla F(\theta_t^c,\theta_t^s)\|^2+L\eta^2\mathbb{E}\|[[g_t^c,g_t^s]]-\nabla F(\theta_t^c,\theta_t^s)\|^2
    \nonumber
    \\
    &+
    L\eta^2\mathbb{E}\|[\hat{g_t^c},\hat{g_t^s}]-[g_t^c,g_t^s]\|^2
    -
    \eta\mathbb{E}\langle\nabla F(\theta_t^c,\theta_t^s), [\hat{g_t^c},\hat{g_t^s}]-[g_t^c,g_t^s]\rangle
    \label{443}
    \\
    &\le
    (L\eta^2-\eta)\mathbb\|\nabla F(\theta_t^c,\theta_t^s)\|^2+L\eta^2\sigma^2+L\eta^2(H^2E^2_t+H^2J^2E^2_t)
    \nonumber
    \\
    &-
    \eta\mathbb{E}\langle\nabla F(\theta_t^c,\theta_t^s), [\hat{g_t^c},\hat{g_t^s}]-[g_t^c,g_t^s]\rangle,
    \label{444}
\end{align}
where Eq.~\eqref{441} comes from the formula for stochastic gradient descent $\theta^s_{t+1} = \theta^s_{t} - \eta_s\hat{g_t^s}$, Eq.~\eqref{442} stems from $\mathbb{E}\| a +  b \|^2 \le 2\mathbb{E}\| a \|^2 + 2\mathbb{E} \| b \|^2$, 
Eq.~\eqref{443} holds according to $\mathbb{E}\| g_t^s\|^2 = \mathbb{E}\| g_t^s-\nabla f^s(\theta_t^s)\|^2 + \mathbb{E}\| \nabla f^s(\theta_t^s)\|^2$,
Eq.~\eqref{444} is derived based on Assumption~\ref{ass_var} and Lemma~\ref{gradientbound}.

If for any $t$, the following inequality is satisfied:
\begin{align}
    \label{product}
    \mathbb{E}\langle\nabla F(\theta_t^c,\theta_t^s), [\hat{g_t^c},\hat{g_t^s}]-[g_t^c,g_t^s]\rangle\ge0.
\end{align}
Thus, we can derive: 
\begin{align}
    \mathbb{E}(F(\theta_{t+1}^c,\theta_{t+1}^s)-F(\theta_t^c,\theta_t^s))
    &\le
    (L\eta^2-\eta)\mathbb\|\nabla F(\theta_t^c,\theta_t^s)\|^2+L\eta^2\sigma^2+L\eta^2(H^2E^2_t+H^2J^2E^2_t)
    \\
    &\le
    -\frac{\eta}{2}\mathbb\|\nabla F(\theta_t^c,\theta_t^s)\|^2+L\eta^2\sigma^2+L\eta^2(H^2E^2_t+H^2J^2E^2_t),
\end{align}
where $L\eta^2-\eta\le-\frac{\eta}{2}$ holds because $\eta\in(0, \frac{1}{2L}]$.

Rearrange the inequality, divide $\frac{\eta}{2}$ on both sides to get:
\begin{align}
    \mathbb{E}\|\nabla F(\theta_t^c,\theta_t^s)\|^2
    \le
    \frac{2}{\eta}\mathbb{E}(F(\theta_{t}^c,\theta_{t}^s)-F(\theta_{t+1}^c,\theta_{t+1}^s))+2L\eta(\sigma^2 + H^2E_t^2+H^2J^2E_t^2).
\end{align}

Accumulate $T$ iterations, divide $T$ on both sides, then:
\begin{align}
    \frac{1}{T}\sum_{t=1}^{T}\mathbb{E}\|
    \nabla F(\theta_t^c,\theta_t^s)
    \|^2\le\frac{2}{\eta T}\mathbb{E}(F(\theta_{1}^c,\theta_{1}^s)-F(\theta_{T+1}^c,\theta_{T+1}^s))
    +
    2L\eta(\sigma^2+(1+J^2)H^2\frac{1}{T}\sum_{t=1}^{T}E_t).
\end{align}

Let $\varepsilon^2=(1+J^2)H^2\frac{1}{T}\sum_{t=1}^{T}E_t$ and $\gamma = \mathbb{E}(F(\theta_{1}^c,\theta_{1}^s)-F(\theta_{T+1}^c,\theta_{T+1}^s))$ then:
$$
    \frac{1}{T}\sum_{t=1}^{T}\mathbb{E}\|\nabla F(\theta_t^c,\theta_t^s)\|^2
    \le
    \frac{2\gamma}{\eta T}+2L\eta(\sigma^2+\varepsilon^2).
$$
Choosing $\eta = \sqrt{\frac{\gamma}{TL(\sigma^2+\varepsilon^2)}}$, 
we have the following convergence rate:
\begin{align*}    
    \frac{1}{T}\sum_{t=1}^{T}\mathbb{E}\|\nabla F(\theta_t^c,\theta_t^s)\|^2\le 4\sqrt{\frac{\gamma L(\sigma^2+\varepsilon^2)}{T}}.
\end{align*}

If Eq.\eqref{product} is not satisfied, we have:
\begin{align}
    -\eta\mathbb{E}\langle\nabla F(\theta_t^c,\theta_t^s), [\hat{g_t^c},\hat{g_t^s}]-[g_t^c,g_t^s]\rangle
    &\le
    \frac{\eta}{2}\mathbb{E}\|\nabla F(\theta_t^c,\theta_t^s)\|^2
    +
    \frac{\eta}{2}\|[\hat{g_t^c},\hat{g_t^s}]-[g_t^c,g_t^s]\|^2
    \label{445}
    \\
    &\le
    \frac{\eta}{2}\mathbb{E}\|\nabla F(\theta_t^c,\theta_t^s)\|^2
    +
    \frac{\eta}{2}(H^2E^2_t+H^2J^2E^2_t),
    \label{446}
\end{align}
where Eq.~\eqref{445} comes from $-\langle a, b \rangle=\frac{1}{2}\| a\|^2 + \frac{1}{2}\| b\|^2 -\frac{1}{2}\| a + b\| ^2$, Eq.~\eqref{446} is derived based on Lemma.\ref{gradientbound}.

\begin{align}
    \mathbb{E}(F(\theta_{t+1}^c,\theta_{t+1}^s)-F(\theta_t^c,\theta_t^s))
    &\le
    (L\eta^2-\eta)\mathbb{E}\|\nabla F(\theta_t^c,\theta_t^s)\|^2+L\eta^2(\sigma^2+H^2E^2_t+H^2J^2E^2_t)
    \nonumber
    \\
    &+\frac{\eta}{2}\mathbb{E}\|\nabla F(\theta_t^c,\theta_t^s)\|^2
    +
    \frac{\eta}{2}(H^2E^2_t+H^2J^2E^2_t)
    \\
    &=(L\eta^2-\frac{\eta}{2})\mathbb{E}\|\nabla F(\theta_t^c,\theta_t^s)\|^2+L\eta^2(\sigma^2+H^2E^2_t+H^2J^2E^2_t)+\frac{\eta}{2}(H^2E^2_t+H^2J^2E^2_t)
    \\
    &\le
    -\frac{\eta}{4}\mathbb{E}\|\nabla F(\theta_t^c,\theta_t^s)\|^2+L\eta^2(\sigma^2+H^2E^2_t+H^2J^2E^2_t)+\frac{\eta}{2}(H^2E^2_t+H^2J^2E^2_t)
    \label{447}
\end{align}

Where Eq.~\eqref{447} stems from $L\eta^2-\frac{\eta}{2}\le-\frac{\eta}{4}$ because $\eta\in(0, \frac{1}{4L}]$.

Rearrange the inequality, divide $\frac{\eta}{4}$ on both sides to get: 
\begin{align}
    \mathbb{E}\|\nabla F(\theta_t^c,\theta_t^s)\|^2
    \le
    \frac{4}{\eta}\mathbb{E}(F(\theta_{t}^c,\theta_{t}^s)-F(\theta_{t+1}^c,\theta_{t+1}^s))+4L\eta(\sigma^2 + H^2E_t^2+H^2J^2E_t^2)+2(H^2E_t^2+H^2J^2E_t^2).
\end{align}

Accumulate $T$ iterations, divide $T$ on both sides, then:
\begin{align}
\frac{1}{T}\sum_{t=1}^{T}\mathbb{E}\|\nabla F(\theta_t^c,\theta_t^s)\|^2&\le\frac{4}{\eta T}\gamma+4L\eta(\sigma^2+\varepsilon^2)+2\varepsilon^2.
\end{align}

Choosing $\eta_s=\sqrt{\frac{\gamma}{TL(\sigma^2+\varepsilon^2)}}$, we can derive the following convergence rate:

\begin{align}
\frac{1}{T}\sum_{t=1}^{T}\mathbb{E}\|\nabla F(\theta_t^c,\theta_t^s)\|^2\le 8\sqrt{\frac{\gamma L(\sigma^2+\varepsilon^2)}{T}}+2\varepsilon^2.
\end{align}
\end{proof}

\section{Mask vs. Key-Value Pair}
\label{appendix:mask_vs_keyvalue}
Sparse matrices have multiple storage methods, and the key-value pair is one of them. When the sparsification ratio is high, the cost of the key-value pairs is lower. We use $d$ to represent the data dimension, $f_1$ to represent the bit width of the floating-point number, and $f_2$ to represent the bit width of the key. Therefore, the storage cost of the 1-bit mask is $d+f_1k$, and the storage cost of key-value pairs is $(f_1+f_2)k$. When the sparse rate $1-k/d$ is greater than $1-1/f_2$, the overhead of the key-value pair is lower. Because the dimensions of the feature map are usually $4$ (batch, channel, height, width), the bit width of $f_2$ usually needs to be at least $32$. Therefore, key-value pairs should only be used for storage when the sparse rate exceeds $96.875\%$.

\section{Proof of Theorem~\ref{better_than_qu_sp}}
\label{appendix:better_than_qusp}
\begin{proof}
Here are the upper bounds of the error for QU and MS:
\begin{align}
    \mathbb{E}\|QU(x)-x\|_2^2&\le\frac{\sqrt{d}}{2^{q_1}-1}\|x\|_2^2,
    \\
    \mathbb{E}\|MS(x)-x\|_2^2&\le\frac{\alpha\sqrt{d-k_2}}{2^{q_2}-1}\|x\|_2^2.
\end{align}
The condition for MS to be superior to QU is:
\begin{align}
    \frac{\alpha\sqrt{d-k_2}}{2^{q_2}-1}\le\frac{\sqrt{d}}{2^{q_1}-1}.
    \label{better_qu_target}
\end{align}
Rearrange the inequality, we can get:
\begin{align}
    \alpha\sqrt{1-\frac{k_2}{d}}\le\frac{2^{q_2}-1}{2^{q_1}-1}.
\end{align}
The right formula satisfies:
\begin{align}
    \frac{2^{q_2}-1}{2^{q_1}-1}>\frac{2^{q_2}-1}{2^{q_1}}=2^{-1-f\frac{k_2}{d}}.
\end{align}
Note that $\alpha\in(0, 0.5)$ and $\frac{k_2}{d}\rightarrow0$, it is obtained that $\alpha\sqrt{1-\frac{k_2}{d}}\le2^{-1-f\frac{k_2}{d}}$, and Eq.~\eqref{better_qu_target} is satisfied. 

Here are the upper bounds of the error for SP and MS:
\begin{align}
    \mathbb{E}\|SP(x)-x\|_2^2&\le\frac{d-k_1}{d}\|x\|_2^2,
    \\
    \mathbb{E}\|MS(x)-x\|_2^2&\le\frac{\alpha\sqrt{d-k_2}}{2^{q_2}-1}\|x\|_2^2.
\end{align}
The condition for MS to be superior to SP is:
\begin{align}
    \frac{\alpha\sqrt{d-k_2}}{2^{q_2}-1}\le\frac{d-k_1}{d}.
    \label{target_better_sp}
\end{align}
The left formula satisfies:
\begin{align}
    \frac{\alpha\sqrt{d-k_2}}{2^{q_2}-1}\le  \frac{\alpha (d-k_2)}{2^{q_2}-1}\le\alpha\exp{\frac{k_1-k_2}{q_2-1}}=\alpha\exp\frac{d}{f}.
\end{align}
Rearrange the inequality:
\begin{align}
    \alpha\exp{\frac{d}{f}}\le\frac{2^{q_2}-1}{d}.
\end{align}
Note that the left formula is an constant, as the $q_2$ increase, the inequality can be satisfied, and Eq.~\eqref{target_better_sp} can be satisfied. 

Therefore, Theorem.~\ref{better_than_qu_sp} is proven.
\end{proof}
\end{document}